\documentclass[lettersize,journal]{IEEEtran}

\usepackage{cite}
\usepackage{amsmath,amssymb,amsfonts,amsthm}
\usepackage{algorithm}
\usepackage{algorithmic}
\usepackage{textcomp}
\usepackage{graphicx}
\usepackage{textcomp}
\usepackage{xcolor}
\usepackage{cite}
\usepackage{cases}
\usepackage{mathrsfs}
\usepackage{amscd}
\usepackage{graphicx}  
\usepackage{epstopdf}
\usepackage{subfigure}
\usepackage{diagbox}
\usepackage{multirow}
\usepackage{booktabs}
\usepackage{hyperref}
\renewcommand \thetable{\arabic{table}}
\usepackage{epstopdf}
\usepackage{bm}
\usepackage{booktabs}
\usepackage{multirow}
\usepackage{float}
\usepackage{stfloats}
\allowdisplaybreaks

\newcommand{\error}{\operatorname{error}}
\newcommand{\TD}{\operatorname{TD}}
\newtheorem{theorem}{Theorem}
\newtheorem{lemma}{Lemma}
\begin{document}

\title{Multi-State TD Target for Model-Free Reinforcement Learning}

\author{Wuhao Wang, Zhiyong Chen, and Lepeng Zhang
\thanks{W. Wang and Z. Chen are with the School of Engineering, The University of Newcastle, Callaghan, NSW 2308, Australia. L. Zhang is with the Department of Computer and Information Science, The Linköping University, 581 83 Linköping, Sweden.}
}



\maketitle

\begin{abstract}
Temporal difference (TD) learning is a fundamental technique in reinforcement learning that updates value estimates for states or state-action pairs using a TD target. This target represents an improved estimate of the true value by incorporating both immediate rewards and the estimated value of subsequent states. Traditionally, TD learning relies on the value of a single subsequent state. We propose an enhanced multi-state TD (MSTD) target that utilizes the estimated values of multiple subsequent states. Building on this new MSTD concept, we develop complete actor-critic algorithms that include management of replay buffers in two modes, and integrate with deep deterministic policy optimization (DDPG) and soft actor-critic (SAC). Experimental results demonstrate that algorithms employing the MSTD target significantly improve learning performance compared to traditional methods.The source code used for our experiments can be found on GitHub\footnote{\href{https://github.com/WuhaoStatistic/MSTD-Multi-State-TD-Target-for-Model-Free-Reinforcement-Learning}{https://github.com/WuhaoStatistic}}
\end{abstract}


\section{Introduction}

\IEEEPARstart{D}{eep} reinforcement learning consistently grapples with various challenges, notably inaccurate Q-function estimation and delayed rewards. These obstacles often result in suboptimal policies and sluggish convergence rates, posing significant hurdles to the efficacy of reinforcement learning algorithms. To address the former challenge, researchers have proposed diverse techniques within existing algorithmic frameworks. For instance, off-policy algorithms like deep deterministic policy gradient (DDPG) employ target networks to stabilize value estimations during training. Similarly, soft actor-critic (SAC) integrates entropy regularization to foster exploration and prevent premature convergence, while twin delayed DDPG (TD3) introduces twin critics and policy smoothing to enhance stability and sample efficiency. As for the latter challenge, researchers have advocated for multi-step reinforcement learning, which accounts for returns over multiple time steps. This approach better accommodates delayed rewards, thereby enhancing the agent's performance within the environment.


In model-based reinforcement learning, multi-step greedy reinforcement learning is widely employed, leveraging the environment model to predict future states and inform action selection \cite{Model-BasedValueEstimationforEfficientModel-FreeReinforcementLearning}. Conversely, in model-free reinforcement learning, numerous multi-step methods compute Q-values by considering rewards and final actions across extended trajectories. These approaches are applied in various domains, including image processing and energy management, showcasing their versatility and effectiveness in improving learning efficiency and performance \cite{FullyConvolutionalNetworkwithMulti-StepReinforcementLearningforImageProcessing}, \cite{Multi-stepreinforcementlearningformodel-freepredictiveenergymanagementofnelectrifiedoff-highwayvehicle}, \cite{TheEffectofMulti-stepMethodsonOverestimationinDeepReinforcementLearning}, \cite{Multi-stepreinforcementlearning-basedoffloadingforvehicleedgecomputing}.

Existing multi-step reinforcement learning algorithms typically focus on processing reward signals while considering the Q-function of a single subsequent state. To overcome this limitation, we introduce a novel approach, multi-state TD target (MSTD), which enables the agent to simultaneously integrate multi-step reward signals and Q-values from various states, thereby enhancing performance. To facilitate updates based on MSTD, we propose two modes for managing the replay buffer: action-loaded and action-generated. The former emphasizes past experiences, aiding Q-value accuracy in fitting the data within the replay buffer, while the latter permits the agent to reselect actions from past experiences, ensuring Q-value estimation aligns more closely with the current policy.

We integrated MSTD into DDPG and SAC and conducted experiments in virtual environments, including Walker-v4 and HalfCheetah-v4\cite{gym}. The results demonstrate a significant improvement in learning performance. Our contributions can be summarized as follows:
 
\begin{enumerate}
    \item Introduction of the MSTD framework, incorporating the multi-state replay buffer and multi-state objective functions. We proved the convergence of MSTD framework, this work also shows the convergence of multi-step reinforcement learning, to our knowledge, this is the first work that proves the multi-step reinforcement learning can not converge to the optimal Q-function.
    \item Theoretical and experimental analysis of alignment with actor-critic algorithms and the maximum entropy framework.
    \item Introduction of two modes: action-loaded and action-generated, with experimental comparisons between them.
\end{enumerate}


\section{Related Work}

Kristopher \cite{Multistepreinforcementlearningaunifyingalgorithm} introduced a method by incorporating \(Q_\phi(\sigma)\) into function approximation to sample varying degrees between Sarsa and Expected Sarsa. Experimental results demonstrate that intermediate values of \(\sigma\) yield better results compared to both algorithms. However, the experiments in this study are limited to on-policy prediction and control problems.

Rainbow \cite{rainbow} integrates existing advancements in reinforcement learning, including multi-step settings, arguing that multi-step learning not only enhances final rewards but also accelerates learning in the early stages. However, their research primarily focuses on discrete action spaces, whereas our work investigates continuous control tasks.

Lingheng \cite{TheEffectofMulti-stepMethodsonOverestimationinDeepReinforcementLearning} explores multi-step reinforcement learning in continuous control tasks by combining multi-step settings with DDPG. However, their algorithm only aggregates rewards from multiple steps, neglecting intermediate information from middle state-action pairs, thus adhering to conventional multi-step reinforcement learning.

Yuan and Yu \cite{averagemulti} propose averaging \(n\) standard return functions, each representing a sum of discounted reward signals from different time intervals. Their method mitigates reward hacking by emphasizing further rewards to the agent. Their experiments focus on grid worlds with simple discrete actions and states, and their model's horizon is constrained to a single state.
 
Multi-step settings are also applied to model-based algorithms. Model-based value expansion (MVE) \cite{Model-BasedValueEstimationforEfficientModel-FreeReinforcementLearning} is a multi-step method that applies multi-step learning on a learned environment model. Although the algorithm is model-free, the environment model requires additional calculations, and accurately modeling a realistic environment can be challenging. 

Enhanced algorithms like stochastic ensemble value expansion (STEVE) \cite{Sample-efficientreinforcementlearningwithstochasticensemblevalueexpansion} extend various multi-step Q-value functions to achieve better performance in complex environments. However, using ensemble environment models, which include transition dynamics and reward functions, involves extra computation. Model-based planning policy learning with multi-step plan value estimation (MPPVE) \cite{Multi-StepPlanValueEstimation} collects k-step model rollouts, where the start state is real and others are generated by the model. However, only the gradient on the start state is computed, reducing computation pressure and alleviating model error. This approach, however, cannot be applied to model-free algorithms as they require an environment model.

\section{Preliminaries}

We consider a standard RL setup according to \cite{suttonbook}. At each timestep $t$, the agent receives an observation $\mathbf{s}_t$ from the environment, takes an action $\mathbf{a}_t$, and obtains a scalar reward $\mathbf{r}_{t+1}$ together with the next observation $\mathbf{s}_{t+1}$. For an infinite Markov decision process (MDP) defined by $(\mathcal{S}, \mathcal{A}, \mathcal{P}, \mathcal{R}, \gamma)$, $\mathcal{S}$, $\mathcal{A}$, and $\mathcal{R}$ are the sets of $\mathbf{s}_t$, $\mathbf{a}_t$, and $\mathbf{r}_{t+1}$ respectively. $\mathcal{P}: \mathcal{S} \times \mathcal{A} \times \mathcal{S} \rightarrow \mathbb{R}$ is the transition probability function. $\gamma \in (0,1]$ is the discount factor that balances the future and current reward. 

The accumulated expected reward $\mathbf{G}_T$:
\begin{align}
\mathbf{G}_T=\sum_{t=0}^{T} \gamma^t \mathbf{r}_{t},
\end{align}
where $T$ is the final time step in an episodic task and $T=\infty$ when a task is continuous. 
An agent's behavior is defined by a policy $\pi$, which maps the observation $\mathbf{s}_t$ to action $\mathbf{a}_t$. 
The state-action value function $Q^{\pi}(s,a)$ defined below:
\begin{align}
\label{eq:preli}
Q^{\pi}(s, a) &= \underset{\mathbf{s}_t \sim \mathcal{P}, \mathbf{a}_t \sim \pi}{\mathbb{E}}[\mathbf{r}_{t+1}+\gamma\mathbf{r}_{t+2}+\cdots \mid s=\mathbf{s}_t, a=\mathbf{a}_t].
\end{align}
According to the algorithms, the behavior policy can be deterministic or stochastic. An RL agent aims to find the optimal policy $\pi^*$ to maximize $Q^{\pi}(s,a)$.

Under policy gradient and actor-critic style, the behavior policy is defined as actor $\pi_{\theta}$, parameterized by $\theta$. Normally, a deterministic actor will output a scalar for each dimension of action, whereas a stochastic actor will generate the parameters of a certain distribution to sample an action. The critic, parameterized by $\phi$ and denoted by $Q^{\pi}_{\phi}(s,a)$, estimates the state-action value function $Q^{\pi}(s,a)$. 

In the following text, we will use $Q^{\pi}_{\phi}(\mathbf{s}_t,\mathbf{a}_t)$ to indicate the state-action value function at time $t$.  The temporal difference (TD) method, which allows the gradient update to happen when a task is not finished, is applied to accelerate convergence and reduce variance, the corresponding TD target is:
\begin{align}
\label{eq:tddef}
\TD &= \mathbf{r}_{t+1}+\gamma Q^{\pi}_{\phi}(\mathbf{s}_{t+1},\mathbf{a}_{t+1}).
\end{align}
This target determines the state-action function by considering its value at the subsequent one-step and the intermediate reward.
It implies a TD error, $\TD -Q^{\pi}_{\phi}(\mathbf{s}_t,\mathbf{a}_t)$, to update the value function as follows:
\begin{align}  \label{valueupdate}
Q^{\pi}_{\phi}(\mathbf{s}_t,\mathbf{a}_t)\leftarrow
(1-\alpha) Q^{\pi}_{\phi}(\mathbf{s}_t,\mathbf{a}_t)  + \alpha \TD.
\end{align}
for a learning rate $\alpha$.

Recent literature extends this one-step TD target into a multi-step rule, as shown below:
\begin{align}
\label{eq:multitddef}
  \TD &= \sum_{i=1}^L \gamma^{i-1}\mathbf{r}_{t+i}+\gamma^L Q^{\pi}_{\phi}(\mathbf{s}_{t+L},\mathbf{a}_{t+L}),
\end{align}
where $L$ is the step size. When $L=1$, the multi-step method degrades to single-step reinforcement learning. For an episodic task or task with limited steps, where the maximum step $T$ exists,  the multi-step update rule for time interval $T-L<t\leq T$  is different. If  $t>T-L$,  a truncated return function is applied:
\begin{align}
\label{eq:truncatedmultitddef}
   \TD &= \mathbf{r}_{t+1}+\gamma\mathbf{r}_{t+2}+\cdots+\gamma^{T-t-1}\mathbf{r}_{T-1}.
\end{align}
This could be done by setting the element in \eqref{eq:multitddef} whose time index is greater than $L$ to zero. Meanwhile, this update can not happen when the length of the collected trajectory is shorter than $L$.
It is noted that in some literature such as \cite{averagemulti}, the multi-step intermediate reward $\sum_{i=1}^L \gamma^{i-1}\mathbf{r}_{t+i}$ in the definition of \eqref{eq:multitddef} takes different weighted summation form. 

\section{Multi-state Temporal Difference}
 
In both the single-step TD \eqref{eq:tddef} and the multi-step version \eqref{eq:multitddef}, the TD target is calculated based on the Q-value at one state, either $\mathbf{s}_{t+1}$ or $\mathbf{s}_{t+L}$, and the corresponding action. Therefore, they are referred to as single-state TD methods. The proposed TD target is based on the Q-value at multiple states. The idea is motivated by averaging the TD targets defined for various step sizes, that is,
\begin{align}
\label{eq:multistateqsum}
\TD &=\frac{1}{L}\sum_{l=1}^L \left\{ \sum_{i=1}^l \gamma^{i-1}\mathbf{r}_{t+i}+\gamma^l Q^{\pi}_{\phi}(\mathbf{s}_{t+l},\mathbf{a}_{t+l}) \right\} \nonumber \\
&=\frac{1}{L}\sum_{l=1}^L \sum_{i=1}^l \gamma^{i-1}\mathbf{r}_{t+i}+\frac{1}{L}\sum_{l=1}^L \gamma^l Q^{\pi}_{\phi}(\mathbf{s}_{t+l},\mathbf{a}_{t+l}).  
\end{align}
This multi-state TD (MSTD) target requires the $Q$-values at multiple states $\mathbf{s}_{t+l}$, $l=1,\cdots,L$, and the corresponding actions. We aim to study how this new MSTD can be used for value function update as in \eqref{valueupdate} and its advantages compared with the conventional \eqref{eq:tddef} or \eqref{eq:multitddef}.



The objective for training the neural network is defined as:
\begin{align}
\label{eq:msobj}
\min_\phi J(\phi) = \frac{1}{N}\sum_{i=1}^N \left[ Q^{\pi}_{\phi}(\mathbf{s}^i_t,\mathbf{a}^i_t)-\TD^i \right]^2,
\end{align}
where each pair $(\mathbf{s}^i_t,\mathbf{a}^i_t)$ is a sample from the replay buffer, $\TD^i$ is the corresponding TD target, $i$ indicates the sample index, and $N$ is the batch size. Since the convergence of conventional n-step reinforcement learning is guaranteed, our method based on the multi-state TD target can also converge to the optimal policy, noting that the gradient in this case is the average of the gradients in conventional multi-step TD, as derived from \eqref{eq:multistateqsum}.

Using the conventional multi-step (single-state) TD target defined in \eqref{eq:multitddef}, we require $(\mathbf{s}_t,\mathbf{a}_t,\mathbf{r}_{t+1}, \cdots, \mathbf{r}_{t+L}, \mathbf{s}_{t+L},\mathbf{a}_{t+L})$ for each sample $i$ to complete the calculation \eqref{eq:msobj}. Since $\mathbf{a}_{t+L}$ is generated by the actor according to the current policy, the sample data in the replay buffer takes the tuple structure $(\mathcal{S}, \mathcal{A}, \mathcal{R}_1, \mathcal{R}_2,\cdots, \mathcal{R}_L, \mathcal{S}_{\operatorname{next}})$, where $\mathcal{S}_{\operatorname{next}}$ indicates the next state that involves in TD calculation (if $\mathcal{S}$ is recorded at time step $t$, then $\mathcal{S}_{\operatorname{next}}$ will be recorded at time step $t+L$). The design of a replay buffer and the corresponding training algorithms for the new MSTD in \eqref{eq:multistateqsum} becomes more complex as it requires more state/action pairs to complete the calculation in \eqref{eq:msobj}. These aspects will be discussed in the subsequent section.


\section{Algorithms}

In this section, we investigate the design of a replay buffer and a training scheme for MSTD within actor-critic settings. We introduce two types of modes: action-loaded and action-generated. To illustrate these algorithms, we provide pseudocode based on DDPG with MSTD as a case study. Additionally, we discuss the integration of the soft actor-critic algorithm with MSTD.

 \subsection{Maintenance of a Replay Buffer}
\label{se:multistatetrajecandbuffer}

The MSTD method, as described in \eqref{eq:multistateqsum}, takes into account the Q-values of state-action pairs over an extended trajectory. To efficiently compute \eqref{eq:msobj} using MSTD, we need to store the sequences 
\begin{align} \label{mstdsequence}
 (\mathbf{s}_t,\mathbf{a}_t,\mathbf{r}_{t+1}, \cdots,
\mathbf{s}_{t+L-1},\mathbf{a}_{t+L-1},\mathbf{r}_{t+L}, \mathbf{s}_{t+L},\mathbf{a}_{t+L})   
\end{align}
for each sample $i$ in the replay buffer. This storage is necessary to capture the detailed interaction information between the agent and the environment over the trajectory.

Since $\mathbf{a}_{t+L}$ is generated by the actor according to the current policy, the sample data in the replay buffer follows the structure $(\mathcal{S}_0, \mathcal{A}_0, \mathcal{R}_1, \cdots, \mathcal{S}_{L-1}, \mathcal{A}_{L-1}, \mathcal{R}_L,\mathcal{S}_{L})$.  To maintain and update samples in this fixed-length structure, we continuously collect tuples $(\mathcal{S},\mathcal{A},\mathcal{R},\mathcal{S}_{\operatorname{next}})$ during training. Each time a new tuple is collected, we delete the oldest triplet $(\mathcal{S}_0, \mathcal{A}_0, \mathcal{R}_1)$ from the sample and append the new tuple to the end with $\mathcal{S}=\mathcal{S}_{L}$. This process ensures that the sample is continuously updated. Figure \ref{fig:tempbuffer} illustrates this process. 
 
It is noteworthy that when the agent's episode terminates at time $T$, no more tuples are collected. In this scenario, we repeatedly append the last tuple, known as the termination tuple, to the sample using the same updating process. This repetition continues for $L$ updates until the termination tuple fills the sample. These repetitive tuples are marked with a termination flag, indicating that they do not contribute to the gradient computation. This operation helps the agent learn from the final states of the episode, as mentioned in \eqref{eq:truncatedmultitddef}.

\begin{figure}[H]
\centering
\includegraphics[width=0.9\linewidth]{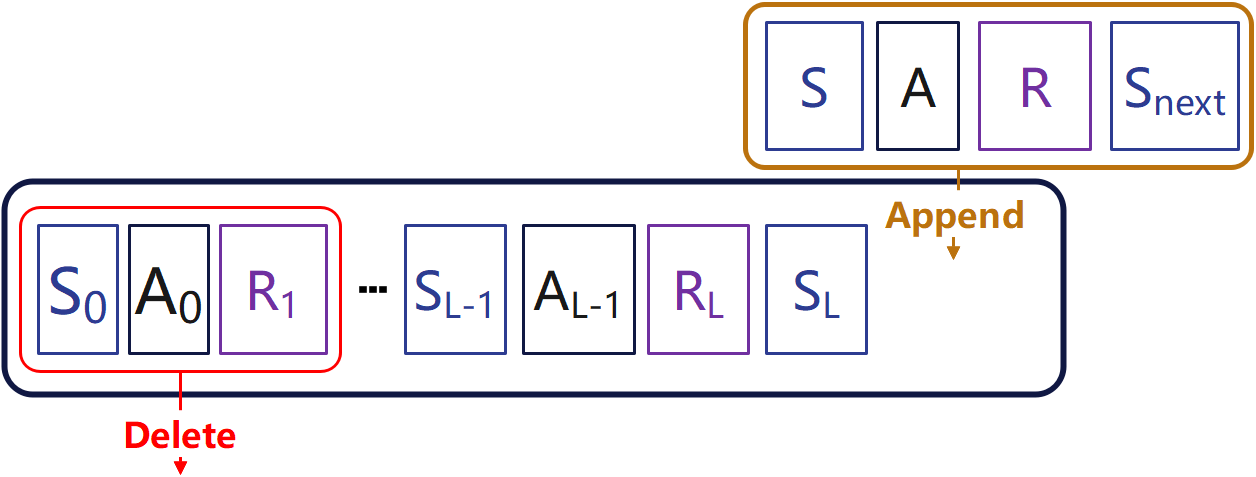}
\caption{Management of sample structure in a replay buffer.}
\label{fig:tempbuffer}
\end{figure}

\subsection{Action-loaded and Action-generated Modes}

In the sequences \eqref{mstdsequence}, it is observed that the action $\mathbf{a}_{t+L}$ is generated by the actor according to the current policy, while the intermediate actions $\mathbf{a}_{t+1},\cdots, \mathbf{a}_{t+L-1}$ are saved in the buffer and loaded for the calculation of \eqref{eq:msobj}. This operation is called an action-loaded mode. Therefore, the average of the Q-values in \eqref{eq:multistateqsum} is specifically expressed as:
\begin{align}
\label{eq:action_lo_td}
\frac{1}{L}\left\{ \sum_{l=1}^{L-1} \gamma^l Q^{\pi}_{\phi}(\mathbf{s}_{t+l},\mathbf{a}_{t+l})
+\gamma^{L} Q^{\pi}_{\phi}(\mathbf{s}_{t+L},\pi_A(\mathbf{s}_{t+L})) \right\}
\end{align}
where $\pi_A(\mathbf{s}_{t+L})$ indicates that this action is generated by the actor according to the policy $\pi_A$.

Alternatively, similar to the action $\mathbf{a}_{t+L}$, the other actions $\mathbf{a}_{t+1},\cdots, \mathbf{a}_{t+L-1}$ can also be generated by the actor, rather than being saved in the buffer. This operation is called an action-generated mode, and the average of the Q-values in \eqref{eq:multistateqsum} is specifically expressed as:
\begin{align}
\label{eq:action_ge_td}
\frac{1}{L}\sum_{l=1}^{L} \gamma^{l} Q^{\pi}_{\phi}(\mathbf{s}_{t+l},\pi_A(\mathbf{s}_{t+l})).
\end{align}
Figure~\ref{fig:action_load_gene} illustrates the structures of these two modes.


\begin{figure}[H]
\centering
\includegraphics[width=0.8\linewidth]{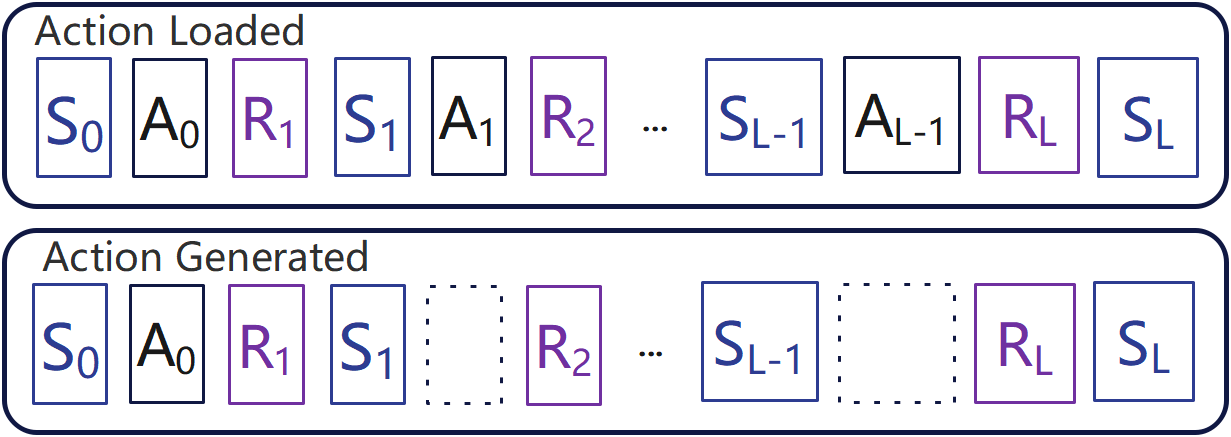}
\caption{Two modes of sample structure in a replay buffer. }
\label{fig:action_load_gene}
\end{figure}

These two modes of operation exhibit different emphases on empiricism and rationalism in multi-state reinforcement learning. The action-loaded mode emphasizes the use of collected data, with the computed results fully reflecting the accumulated experience. On the other hand, the action-generated mode allows the most recent policies to participate in gradient updating. We hypothesize that the policy network will make more accurate decisions as it is trained further. Consequently, this mode leads to the acquisition of Q-functions that are more closely aligned with the true values.

\subsection{Action-loaded DDPG with MSTD}

In this section, we present a DDPG algorithm with action-loaded MSTD as a case study. The algorithm is abbreviated as MSDDPG. The objective for the actor in this algorithm is the same as DDPG:
\begin{align}
    \label{eq:msddpgobactor}
    J_\pi(\theta) = - Q^{\pi}_\phi(\mathbf{s}_t,\pi(\mathbf{s}_t)),
\end{align}
where $Q$ and $\pi$ are the Q-network and policy-network. The details can be found in Algorithm \ref{alg:wddpg}

\begin{algorithm}
 \caption{Action-Loaded MSDDPG}
 \label{alg:wddpg}
 \begin{algorithmic}[1]
 \renewcommand{\algorithmicrequire}{\textbf{Input:}}
 \renewcommand{\algorithmicensure}{\textbf{Output:}}
 \REQUIRE Environment with state space $\mathcal{S}$ and action space $\mathcal{A}$
 \ENSURE  Learned policy $\pi$ and Q-value function $Q$
 
  \STATE Initialize policy network $ \pi(s | \theta)$ and Q networks $Q(s, a | \phi)$ with random weights

  \STATE Initialize target networks $\pi'$ and $Q'$ with weights $\theta' \leftarrow \theta$, $\phi' \leftarrow \phi$\;
  \STATE Initialize multi-state replay buffer $\mathcal{B}$\ 
  
  \STATE Set hyperparameters: discount factor $\gamma$, target smoothing parameter $\tau$, exploration noise $\mathcal{N}$, step size $L$, number of total training episodes $M$, batch size $B$
 
  \FOR {$episode = 1$ to $M$}
        \STATE Receive random initial observation state $\mathbf{s}_0$
        
        \WHILE{Not Terminated}
            \STATE Select action $\mathbf{a}_t = \pi(\mathbf{s}_t | \theta) + \mathcal{N}_t$ according to the current policy and exploration noise;
            \STATE Execute $\mathbf{a}_t$ in the environment and observe next state $\mathbf{s}_{t+1}$ and reward $\mathbf{r}_{t+1}$;
            \STATE Append $(\mathbf{s}_t, \mathbf{a}_t, \mathbf{r}_t, \mathbf{s}_{t+1})$ to $\mathcal{B}$ as illustrated in Figure \ref{fig:tempbuffer};
            \IF{Len($\mathcal{B}$) $\ge$ $B$}
                \STATE Sample a minibatch of $N$ transitions from $\mathcal{B}$;
                \STATE Calculate TD target and objective for Q network according to \eqref{eq:multistateqsum} and \eqref{eq:msobj};
                \STATE Calculate the objective for the policy network according to \eqref{eq:msddpgobactor};
                \STATE update the policy network and the Q network;
                \STATE Update the target networks: $\theta' ;\leftarrow \theta$, $\phi' \leftarrow \phi$
            \ENDIF
  \ENDWHILE
  \STATE Adding final-states-samples to replay buffer as illustrated in Section\ref{se:multistatetrajecandbuffer};
  \ENDFOR   
 \end{algorithmic} 
 \end{algorithm}
 
\subsection{Integration with Maximum Entropy}

The MSTD method can also be integrated with the soft actor-critic (SAC) algorithm. In SAC, the Q-function includes an additional entropy term to encourage exploration, defined as follows:
\begin{equation} Q_{\phi}^{s \pi}\left(\mathbf{s}_t,\mathbf{a}_t\right)=Q^{\pi}_{\phi}\left(\mathbf{s}_t, \mathbf{a}_t\right)-\alpha \log \pi\left(\mathbf{a}_t \mid \mathbf{s}_t\right).
\end{equation}
By incorporating this Q-function into MSTD \eqref{eq:multistateqsum}, we can derive a soft MSTD target.

The SAC algorithm integrated with MSTD, abbreviated as MSSAC, can be implemented in both action-loaded and action-generated modes. In the action-loaded mode, to maintain the consistency of the replay buffer, we do not store the action entropy in the replay buffer. Thus, in practice, action-loaded MSSAC only introduces the entropy term in the last Q-value function. Consequently, the average of the Q-values in \eqref{eq:multistateqsum} is modified to:
\begin{align}
\label{eq:soft_ad_td2}
\frac{1}{L}\left\{ \sum_{l=1}^{L-1} \gamma^l Q^{\pi}_{\phi}(\mathbf{s}_{t+l},\mathbf{a}_{t+l})
+\gamma^{L} Q^{s \pi}_{\phi}(\mathbf{s}_{t+L},\pi_A(\mathbf{s}_{t+L})) \right\}.
\end{align}
In the action-generated mode, the entropy term is included in all the Q-value functions in \eqref{eq:multistateqsum}, and the average of the Q-values is modified to:
 \begin{align}
\label{eq:action_ge_td}
\frac{1}{L}\sum_{l=1}^{L} \gamma^{l} Q^{\pi}_{s \phi}(\mathbf{s}_{t+l},\pi_A(\mathbf{s}_{t+l})).
\end{align}

\section{Experiments}

We conducted experiments in the Gymnasium virtual environment \cite{gym} HalfCheetah-v4 and Walker-v4, using vectorized observations as state inputs. The environment HalfCheetah-v4 involves a two-legged robot moving forward in a 2D plane to maximize speed while keeping stable. The state space is a 17-dimensional vector including the robot's position, velocity, and joint angles. Walker-v4 features a bipedal robot tasked with walking forward in a 2D plane to achieve maximum forward speed. The state space has a 17-dimensional vector that captures the robot's position, velocities, and joint angles. The results are averaged over 15 random seeds. In this section, we aim to illustrate three points:

\begin{enumerate}
    \item  The effectiveness of the multi-state setting compared to original algorithms and conventional multi-step algorithms \eqref{eq:multitddef}.
    \item The comparison between action-loaded and action-generated settings.
    \item The comparison between different step sizes.
\end{enumerate}

The experiment results and the hyperparameter for each experiment can be visited below.

"Hidden width" indicates the number of neurons in the hidden layers. $\tau$ is the proportion in the soft target policy update. "lr" represents the learning rates. $\alpha$ is the entropy regularizer for SAC algorithms while $\mu$, $\theta$, and$\sigma$ is the hyperparameter controlling the random process in DDPG.
\begin{table}[ht]
    \centering
    \begin{minipage}{0.4\linewidth}
        \centering
        \begin{tabular}{cc} 
        \toprule
        hyperparameters & value \\
        \midrule
        hidden width & 256 \\
        $\tau$ & 0.005 \\
        lr & 0.0004  \\
        $\gamma$ & 0.99\\
        batch size & 128\\
        $\alpha$ & 0.12 \\
        \bottomrule
    \end{tabular}
        \caption{Hyperparameters.}
    \end{minipage}%
    \hfill
    \begin{minipage}{0.4\linewidth}
        \centering
        \begin{tabular}{cc} 
        \toprule
        hyperparameters & value \\
        \midrule
        hidden width & 256 \\
        $\tau$ & 0.005 \\
        lr & 0.0003\\
        $\gamma$ & 0.99\\
        batch size & 128\\
        $\mu$ & 0\\
        $\theta$ & 0.2\\
        $\sigma$ & 0.3\\
        \bottomrule
    \end{tabular}
    \caption{Hyperparameters.}
    \end{minipage}
\end{table}

We employed the default hyperparameters outlined in the respective papers for DDPG and SAC, maintaining consistency for MSDDPG and MSSAC across both action-generated and action-loaded settings. Additionally, we compared the results with conventional multi-step algorithms based on the TD target \eqref{eq:multitddef}, denoted as MPDDPG and MPSAC, respectively. To ensure a fair comparison, each model underwent training on the same device for an equal number of training steps.

The empirical results demonstrate that both action-generated and action-loaded MSDDPG, as well as MSSAC, exhibit superior performance compared to both conventional multi-step and original algorithms in the HalfCheetah-v4 environment, as shown in Table~\ref{tab:exp res}. Specifically, in the Walker-v4 environment, the multi-step method outperforms the action-generated multi-state method but is inferior to the action-loaded multi-state method, as illustrated in Figures~\ref{fig:ddpg_walker} and~\ref{fig:sac_walker}. In the HalfCheetah-v4 environment, both action-generated and action-loaded multi-state methods outperform traditional multi-step methods, with the action-generated method achieving the highest performance, as shown in Figures~\ref{fig:ddpg_half} and~\ref{fig:sac_half}.

We also observed that the benefits of using action-loaded and action-generated settings are environment-specific. In the Walker-v4 environment, the action-loaded setting consistently outperforms the action-generated setting in terms of both convergence speed and the reward returned, as shown in Figures~\ref{fig:ddpg_walker} and~\ref{fig:sac_walker}. Conversely, Figures~\ref{fig:ddpg_half} and~\ref{fig:sac_half} clearly show that the action-generated setting outperforms the action-loaded setting in the HalfCheetah-v4 environment.


\begin{table*}[t]
\centering
\caption{Experiment results. AL and AG indicate action-loaded and action-generated methods respectively.}
\begin{tabular}{|c|c|c|c|c|c|c|}
\hline
 & \multicolumn{3}{|c|}{Walker-v4} & \multicolumn{3}{|c|}{HalfCheetah-v4} \\
\hline
SAC & \multicolumn{3}{|c|}{$4476.2\pm707.2$} & \multicolumn{3}{|c|}{$9003.5\pm2846.2$} \\
\hline
MPSAC & $4371.1 \pm 350.2$ & $4506.8 \pm 564.3$ & $4227.1 \pm 556.3$ & $8459.9 \pm 856.7$ & $6478.6 \pm 875.9$ & $5869.2 \pm 942.7$ \\
\hline
MSSAC(AL) & $\mathbf{4938.4\pm341.5}$ & $\mathbf{4953.2\pm503.4}$ & $4504.15\pm1226.3$ & $11447.4\pm1729.9$ &$10574.8\pm1516.2$& $9349.4\pm2494.9$ \\
\hline
MSSAC(AG) & $4456.6\pm844.9$ & $4530.7\pm826.5$ & $\mathbf{4778.4\pm1098.2}$ & $\mathbf{13088.8\pm1134.4}$ & $\mathbf{12168.1\pm1166.2}$ & $\mathbf{11073.3\pm1309.1}$ \\
\hline
 & 2 STEP & 3 STEP & 4 STEP & 2 STEP & 3 STEP & 4 STEP \\
\hline
DDPG & \multicolumn{3}{|c|}{$406.0\pm280.5$} & \multicolumn{3}{|c|}{$5634.5\pm919.0$} \\
\hline
MPDDPG & $1971.0\pm870.3$ & $2853.9\pm908.2$ & $3611.1\pm944.3$ & $5247.9\pm1049.3$ & $4184.6\pm802.4$ & $3759.4\pm475.4$\\
\hline
MSDDPG(AL) & $\mathbf{2756.1\pm1107.5}$ & $\mathbf{4089.2\pm981.1}$ &$\mathbf{4726.0\pm906.0}$& $6623.2\pm1476.4$ & $6085.0\pm1316.1$ & $5338.3\pm1285.6$ \\
\hline
MSDDPG(AG) & $1473.8\pm615.1$ & $2011.6\pm797.6$ & $2424.7\pm1153.3$ & $\mathbf{7717.8\pm624.3}$ & $\mathbf{7069.2\pm1402.1}$ & $\mathbf{7108.33\pm548.0}$ \\
\hline
\end{tabular}
\label{tab:exp res}
\end{table*}

\section{Conclusion}

We propose a multi-state TD target and reinforcement learning algorithms based on this target in two modes: action-loaded and action-generated. Empirical results demonstrate that both modes offer distinct advantages. Future work could explore how different environmental factors influence the performance of these action-loaded and action-generated approaches.

\section*{Appendix: Convergence Proof for Q-learning with MSTD}

In this section, we demonstrate that Q-learning with the MSTD target converges under commonly used conditions. The convergence proof relies on two fundamental lemmas listed below. The proof of Lemma~\ref{lemma:Delta} and the main concept of the proof of Lemma~\ref{lemma:H operator} are detailed in \cite{jaakkola1994convergence} and \cite{melo2001convergence}, respectively.

\begin{lemma}\label{lemma:Delta}
A random process  
\begin{align}
    \Delta_{t+1}(x)=\left(1-\alpha_t(x)\right) \Delta_t(x)+\alpha_t(x) F_t(x)
\end{align}
converges to zero w.p.1 under the following assumptions:
\begin{enumerate}
\item The state space is finite;
    \item $0\leq\alpha_t(x)\leq1$, $\sum_t\alpha_t(x)=\infty$, and $\sum_t\alpha_t^2(x)<\infty$;
    \item $\left\|\mathbb{E}\left[F_t(x)|P_t \right]\right\|_\infty \leq \gamma\left\|\Delta_t\right\|_\infty$ for  $\gamma \in [0,1)$; and
    \item $\operatorname{var}[F_t(x)|P_t] \leq C(1+\|\Delta_t\|_\infty)^2$ for $C>0$,
\end{enumerate}
where $P_t$ stands for the past at step $t$, $a_t$ and $F_t$ are allowed to depend on the past insofar as the above conditions remain valid. The notation $\|\cdot\|_\infty$ indicates the maximum norm and $\left\|\Delta_t\right\|_\infty =
\max_{x} |\Delta_t(x)|$. 
\end{lemma}

\begin{lemma} \label{lemma:H operator}
In the context of multi-state reinforcement learning under a given policy, we consider the Q-function $q: \mathcal{S}\times\mathcal{A}\rightarrow \mathbb{R}$. We define an operator $\mathbf{H}^l$, for a positive integer $l$, as follows:
\begin{align}
    (\mathbf{H}^lq)(s,a) = \sum_{y \in \mathcal{S}} p_l(y|s,a)\left[\bar r_{l}+\gamma^l \max_{b \in \mathcal{A}}q(y,b)\right], \label{Hdef}
\end{align}
where $p_l(y|s,a)$ represents the probability of the transition to state $y$ in $l$ steps given state-action pair $(s,a)$, and $\bar r_l$ is the accumulated reward. This operator is a contraction in the sense that
\begin{align}
    \|\mathbf{H}^l {q_1}-\mathbf{H}^l {q_2}\|_\infty \leq \gamma^l \|q_1-q_2\|_\infty \label{contraction}
\end{align}
for any two Q-functions $q_1$ and $q_2$.
\end{lemma}

\begin{proof} 
Based on the definition \eqref{Hdef} and the maximum norm, we have
    \begin{align*}
        \|&\mathbf{H}^l {q_1}-\mathbf{H}^l {q_2}\|_\infty  \nonumber \\ 
        =& \max_{s\in \mathcal{S},a\in \mathcal{A}} \Bigg| \sum_{y \in \mathcal{S}} p_l(y|s,a)\big[\bar{r}_{l}+\gamma^l \max_{b \in \mathcal{A}}q_2(y,b) -\bar{r}_{l}\nonumber \\
        &-\gamma^l \max_{b \in \mathcal{A}}q_1(y,b) \big]\Bigg| \nonumber \\
        =& \max_{s\in \mathcal{S},a\in \mathcal{A}} \gamma^l \Bigg| \sum_{y \in \mathcal{S}}p_l(y|s,a)\big[\max_{b \in \mathcal{A}}q_1(y,b)-\max_{b \in \mathcal{A}}q_2(y,b)\big]\Bigg|.
\end{align*}
Additionally, we have
\begin{align*}
 & \gamma^l \|q_1-q_2\|_\infty \\
        =& \max_{s\in \mathcal{S},a\in \mathcal{A}} \gamma^l \sum_{y \in \mathcal{S}}p_l(y|s,a)\|q_1-q_2\|_\infty \nonumber \\
        =& \max_{s\in \mathcal{S},a\in \mathcal{A}} \gamma^l \sum_{y \in \mathcal{S}}p_l(y|s,a) \max_{z\in \mathcal{S},b\in \mathcal{A}}|q_1(z,b)-q_2(z,b)| \nonumber \\
       \geq & \max_{s\in \mathcal{S},a\in \mathcal{A}} \gamma^l \sum_{y \in \mathcal{S}}p_l(y|s,a) \Bigg| \max_{b \in \mathcal{A}}q_1(y,b)-\max_{b \in \mathcal{A}}q_2(y,b)\Bigg| .
    \end{align*}
Comparing these two expressions implies \eqref{contraction}.
\end{proof}

In practice, a finite number of training steps and a finite-sized replay buffer are typically used. This results in the sets $\mathcal{S}$ and $\mathcal{A}$ being finite, making $(q,\|\cdot\|_\infty)$ a complete metric space. Furthermore, Lemma~\ref{lemma:H operator} demonstrates that the operator $\mathbf{H}^l$ is a contraction mapping. Consequently, by Banach's Fixed Point Theorem, the operator $\mathbf{H}^l$ has a fixed point. Suppose there exists $Q^f$ such that 
\begin{align}
    \label{eq: fixed point}
  (\mathbf{H}^lQ^f)(s, a)=Q^f(s, a).
\end{align}

Let $Q_t(s,a)$ represent the critic output $Q_{\phi}^{\pi}(s,a)$ with the inputs $s$ and $a$ at gradient step $t$. The action $a$ is either action-loaded or action-generated in the algorithms, which is related to the actor. We separate the actor's behavior from the convergence analysis of the Q-network. Hence, the TD-error corresponding to the MSTD target \eqref{eq:multistateqsum} becomes
\begin{align}
 \TD^{\error} =& \frac{1}{L}\sum_{l=1}^L \Bigg(\sum_{i=1}^l \gamma^{i-1}\mathbf{r}_{t+i}+ \max_{b \in \mathcal{A}} \gamma^l Q_{t}(\mathbf{s}_{t+l},b) \nonumber \\
  &-Q_t(\mathbf{s}_t,\mathbf{a}_t)\Bigg). \label{TDerror}
\end{align}
For such a Q-learning process with the MSTD error \eqref{TDerror}, the operator $\mathbf{H}^l$ is specifically defined as 
\begin{align}
    (\mathbf{H}^lQ_t)(\mathbf{s}_t, \mathbf{a}_t) &= \sum_{y \in \mathcal{S}} p_l(y|\mathbf{s}_t, \mathbf{a}_t)\bigg[\bar{\mathbf{r}}_{t,l}     +\gamma^l \max_{b \in \mathcal{A}}Q_t(y,b)\bigg] \label{HQt}
\end{align}
where
$\bar{\mathbf{r}}_{t,l} =\sum_{i=1}^l \gamma^{i-1}\mathbf{r}_{t+i}$
for simpler notation. The convergence of Q-learning with the MSTD error is stated in the following theorem.

\begin{theorem}\label{theorem1}
For a finite MDP defined by $(\mathcal{S},\mathcal{A},\mathcal{P},\mathcal{R},\gamma)$ where the Q-function $Q_t(s,a)$ is updated by the rule:
    \begin{align}
        \label{eq:theorem target}
        Q_{t+1}\left(\mathbf{s}_t, \mathbf{a}_t\right) = Q_t\left(\mathbf{s}_t, \mathbf{a}_t\right)+\alpha_t \TD^{\error},
    \end{align}
    where $\TD^{\error}$ is defined in \eqref{TDerror}. Then, $Q_t(s,a)$ converges to $Q^f(s,a)$ w.p.1  under the following assumptions:
\begin{enumerate}
    \item $\mathcal{S}$ and $\mathcal{A}$ are finite;
    \item $\alpha_t \in [0,1]$, $\sum_t \alpha_t = \infty$, $\sum_t \alpha_t^2 < \infty$;
    \item $\mathbf{r}_t$ is bounded; and
    \item $\gamma \in [0,1)$.
\end{enumerate}
\end{theorem}

\begin{proof}

With $\TD^{\error}$ defined in \eqref{TDerror}, the equation  \eqref{eq:theorem target} can be rewritten as:
    \begin{align}
        \label{eq:theorem target rewrite}
        Q_{t+1}\left(\mathbf{s}_t, \mathbf{a}_t\right) =& (1-\alpha_t) Q_t\left(\mathbf{s}_t, \mathbf{a}_t\right)
        +\alpha_t \frac{1}{L} \times \nonumber \\
        & \sum_{l=1}^L \left(\bar{\mathbf{r}}_{t,l}  + \max_{b \in \mathcal{A}} \gamma^l Q_{t}(\mathbf{s}_{t+l},b) \right). 
    \end{align}
By subtracting $Q^f(\mathbf{s}_t,\mathbf{a}_t)$ from both side in \eqref{eq:theorem target rewrite}, we can obtain
    \begin{align}
        &Q_{t+1}\left(\mathbf{s}_t, \mathbf{a}_t\right) - Q^f(\mathbf{s}_t,\mathbf{a}_t)\nonumber \\ 
        =& (1-\alpha_t) (Q_t\left(\mathbf{s}_t, \mathbf{a}_t\right)-Q^f(\mathbf{s}_t,\mathbf{a}_t)) +\alpha_t \frac{1}{L}\times  \nonumber \\ 
        & \sum_{l=1}^L \left(\bar{\mathbf{r}}_{t,l}  + \max_{b\in \mathcal{A}} \gamma^l Q_{t}(\mathbf{s}_{t+l},b) 
        -Q^f(\mathbf{s}_t,\mathbf{a}_t)\right). \label{Q-Q*}
    \end{align}
    
Next, we define a function
    \begin{align}
    \label{delta_define}
        \Delta_t(s,a) = Q_t(s,a) -Q^f(s,a).
    \end{align}
As a result, \eqref{Q-Q*} becomes 
\begin{align}
        \label{DeltaProcess}
        \Delta_{t+1}(\mathbf{s}_t,\mathbf{a}_t)  =& (1-\alpha_t)\Delta_t(\mathbf{s}_t,\mathbf{a}_t)+\alpha_t F_t(\mathbf{s}_t,\mathbf{a}_t),
    \end{align}
    where $F_t(\mathbf{s}_t,\mathbf{a}_t)$ is defined as
    \begin{align*}
        F_t(\mathbf{s}_t,\mathbf{a}_t) =&\frac{1}{L}\sum_{l=1}^L \left(\bar{\mathbf{r}}_{t,l}  + \max_{b\in \mathcal{A}} \gamma^l Q_{t}(\mathbf{s}_{t+l},b) 
        -Q^f(\mathbf{s}_t,\mathbf{a}_t)\right).
    \end{align*} 
We will apply Lemma~\ref{lemma:Delta} to \eqref{DeltaProcess} with $x=(\mathbf{s}_t, \mathbf{a}_t)$ to prove that the process converges to zero w.p.1, thereby completing the proof of the theorem.
The first two assumptions of Lemma~\ref{lemma:Delta} are ensured by the corresponding assumptions in the theorem. What remains is to prove assumptions 3) and 4) of Lemma~\ref{lemma:Delta}.

To prove assumption 3), we calculate
    \begin{align*}
        \mathbb{E}\big[F_t(\mathbf{s}_t, \mathbf{a}_t)\big] 
        =&\frac{1}{L}\sum_{l=1}^L \Bigg(\sum_{y \in \mathcal{S}} p_l( y|\mathbf{s}_t,\mathbf{a}_t)\bigg[\bar{\mathbf{r}}_{t,l}\nonumber \\
        &+\gamma^l \max _{b \in \mathcal{A}} Q_t(y, b) -Q^f(\mathbf{s}_t, \mathbf{a}_t)\bigg]\Bigg) \nonumber \\
        =& \frac{1}{L}\sum_{l=1}^L\left[(\mathbf{H}^lQ_t)(\mathbf{s}_t, \mathbf{a}_t) -(\mathbf{H}^l Q^f)(\mathbf{s}_t, \mathbf{a}_t)\right], 
    \end{align*}
where the second equality follows from \eqref{eq: fixed point} and \eqref{HQt}. Next, using the definition of the maximum norm and the property \eqref{contraction}, we have 
    \begin{align}
        \|\mathbb{E}\big[F_t(\mathbf{s}_t, \mathbf{a}_t)\big]\|_\infty \leq & \frac{1}{L}\sum_{i=1}^L \|\mathbf{H}^lQ_t -\mathbf{H}^lQ^f\|_\infty \nonumber \\
        \leq& \frac{1}{L}\sum_{l=1}^L \gamma^l \|Q_t-Q^f\|_\infty
    \end{align}
where $\gamma^l < \gamma <1$. This simplifies to
    \begin{align}
        \|\mathbb{E}\big[F_t(\mathbf{s}_t, \mathbf{a}_t)\big]\|_\infty  
        &\leq \frac{1}{L}\sum_{l=1}^L \gamma \|Q_t-Q^f\|_\infty \nonumber\\
        &=\gamma\|Q_t-Q^f\|_\infty=\|\Delta_t\|_\infty.
    \end{align}
This verifies assumption 3) of Lemma~\ref{lemma:Delta}.

To prove assumption 4), we calculate the variance of $F_t(\mathbf{s}_t,\mathbf{a}_t)$ as follows:
\begin{align*}
&\operatorname{var}[F_t(\mathbf{s}_t,\mathbf{a}_t)]  \\
        =& \mathbb{E}\Bigg[ \bigg(\frac{1}{L}\sum_{l=1}^L\big(\bar{\mathbf{r}}_{t,l} 
        +\gamma^l \max_{b \in \mathcal{A}} Q_t(\mathbf{s}_{t+l},b) -Q^f(\mathbf{s}_t,\mathbf{a}_t)\big) \\
        & -\frac{1}{L}\sum_{l=1}^L(\mathbf{H}^lQ_t)(\mathbf{s}_t,\mathbf{a}_t)     +Q^f(\mathbf{s}_t,\mathbf{a}_t)\bigg)^2\Bigg] \nonumber \\
        =& \frac{1}{L^2} \mathbb{E}\Bigg[\bigg(\sum_{l=1}^L (\bar{\mathbf{r}}_{t,l}+\gamma ^l\max_{b \in \mathcal{A}} Q_t(\mathbf{s}_{t+l},b) ) \\
        &-\sum_{l=1}^L (\mathbf{H}^lQ_t)(\mathbf{s}_t,\mathbf{a}_t)\bigg)^2\Bigg].
        \end{align*}
Using the definition in \eqref{HQt}, we have
\begin{align*}
   (\mathbf{H}^lQ_t)(\mathbf{s}_t, \mathbf{a}_t) &= \mathbb{E} \bigg[\bar{\mathbf{r}}_{t,l}     +\gamma^l \max_{b \in \mathcal{A}}Q_t(\mathbf{s}_{t+l},b)\bigg],
\end{align*}
which can simplify the variance expression as follows:
\begin{align*}
&\operatorname{var}[F_t(\mathbf{s}_t,\mathbf{a}_t)]  \\
               =&\frac{1}{L^2}\mathbb{E}\Bigg[\bigg(\sum_{l=1}^L (\bar{\mathbf{r}}_{t,l}+\gamma ^l\max_{b \in \mathcal{A}} Q_t(\mathbf{s}_{t+l},b) )\nonumber \\
        &-\mathbb{E}\bigg[\sum_{l=1}^L (\bar{\mathbf{r}}_{t,l} 
        +\gamma^l\max_{b \in \mathcal{A}} Q_t(\mathbf{s}_{t+l},b)) \bigg]\bigg)^2\Bigg] \nonumber \\
        =& \frac{1}{L^2}\operatorname{var}\bigg(\sum_{l=1}^L\gamma^l (\bar{\mathbf{r}}_{t,l}+\max_{b \in \mathcal{A}} Q_t(\mathbf{s}_{t+l},b)) \bigg).
        \end{align*}
Applying the variance property yields    
        \begin{align*}
        &\operatorname{var}[F_t(\mathbf{s}_t,\mathbf{a}_t)]  \\ 
        \leq&\frac{1}{L^2}\mathbb{E}\bigg[\bigg(\sum_{l=1}^L\gamma^l (\bar{\mathbf{r}}_{t,l}  +\max_{b \in \mathcal{A}} Q_t(\mathbf{s}_{t+l},b)) \bigg)^2\bigg]\\
         \leq&\frac{\gamma^2}{L^2}\mathbb{E}\bigg[\bigg(\sum_{l=1}^L  (|\bar{\mathbf{r}}_{t,l}|+\max_{b \in \mathcal{A}} |Q^f(\mathbf{s}_{t+l},b)|\\& +\max_{b \in \mathcal{A}} |\Delta_t(\mathbf{s}_{t+l},b)|) \bigg)^2\bigg]\\
          \leq&\gamma^2\bigg(\bar r_{\max}+\max_{s\in \mathcal{A},  b \in \mathcal{A}} |Q^f(s,b)|+\max_{s\in \mathcal{A}, b \in \mathcal{A}} |\Delta_t(s,b)| \bigg)^2\\
          =&\gamma^2\bigg(C_0+\|\Delta_t\|_\infty\bigg)^2
        \end{align*}
where $\bar r_{\max}$ is the maximum value of $\bar{\mathbf{r}}_{t,l}$, which is bounded because $\mathbf{r}_t$ is bounded, and $C_0=\bar r_{\max}+\max_{s\in \mathcal{A},  b \in \mathcal{A}} |Q^f(s,b)|$.

If $C_0 \leq 1$, then
\begin{align*}
        \operatorname{var}[F_t(\mathbf{s}_t,\mathbf{a}_t)]  \leq  \gamma^2\bigg(1+\|\Delta_t\|_\infty\bigg)^2.
        \end{align*}
On the other hand, if $C_0>1$, we have
\begin{align*}
         \operatorname{var}[F_t(\mathbf{s}_t,\mathbf{a}_t)]  \leq   C_0^2 \gamma^2\bigg(1+\|\Delta_t\|_\infty\bigg)^2.
        \end{align*}
In both cases, assumption 4) of Lemma~\ref{lemma:Delta} is  satisfied with $C=\gamma^2$ or $C=C_0^2 \gamma^2$, respectively. 
Thus, the proof is complete. \end{proof}

\bibliographystyle{ieeetr}
\bibliography{a}



\begin{figure*}
\centering
    \begin{minipage}{0.315\linewidth}
	\centering
	\includegraphics[width=\linewidth]{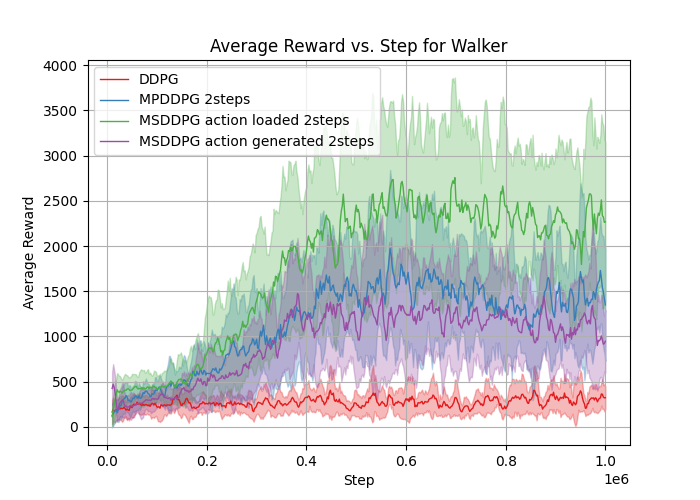}
	\label{fig:ddpgwalker2}
	\end{minipage}
	\hfill
    \begin{minipage}{0.315\linewidth}
	\centering
	\includegraphics[width=\linewidth]{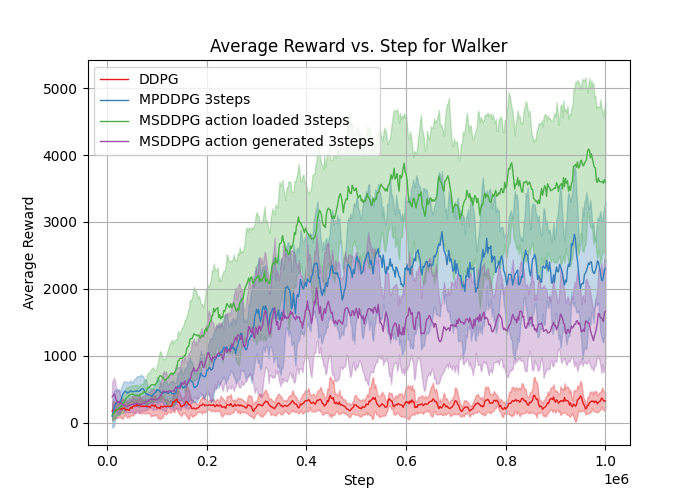}
	\label{fig:ddpgwalker3}
        \end{minipage}
	\hfill
    \begin{minipage}{0.315\linewidth}
	\centering
	\includegraphics[width=\linewidth]{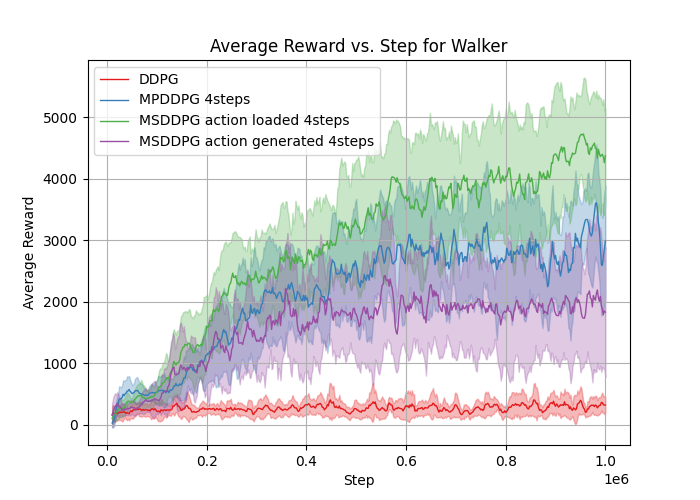}
	\label{fig:ddpgwalker4}
	\end{minipage}
	\caption{Performance comparison of various DDPG algorithms in Walker-v4 environment.}
	\label{fig:ddpg_walker}
\end{figure*}

\begin{figure*}
\centering
    \begin{minipage}{0.315\linewidth}
	\centering
	\includegraphics[width=\linewidth]{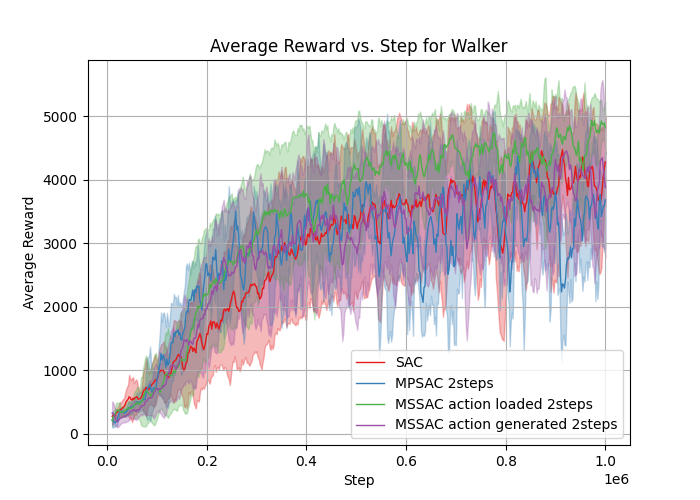}
	\label{fig:sacwalker2}
	\end{minipage}
	\hfill
    \begin{minipage}{0.315\linewidth}
	\centering
	\includegraphics[width=\linewidth]{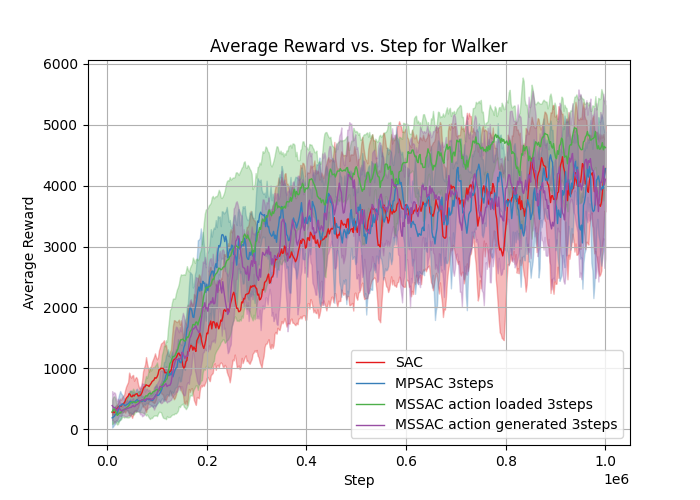}
	\label{fig:sacwalker3}
        \end{minipage}
	\hfill
    \begin{minipage}{0.315\linewidth}
	\centering
	\includegraphics[width=\linewidth]{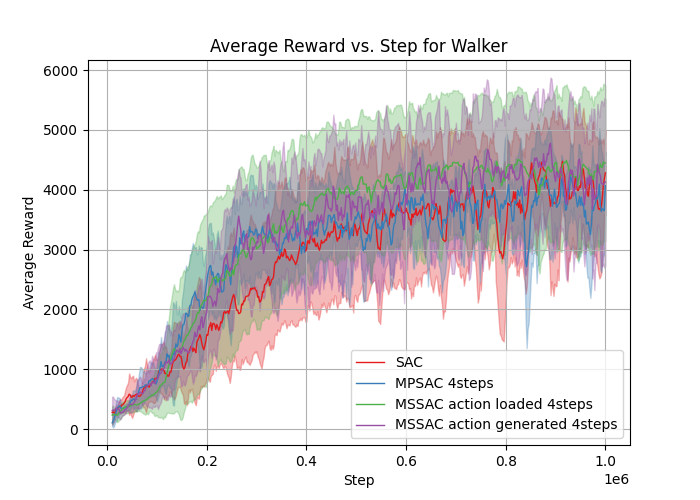}
	\label{fig:sacwalker4}
	\end{minipage}
	\caption{Performance comparison of various SAC algorithms in Walker-v4 environment.}
	\label{fig:sac_walker}
\end{figure*}

\begin{figure*}
\centering
    \begin{minipage}{0.315\linewidth}
	\centering
	\includegraphics[width=\linewidth]{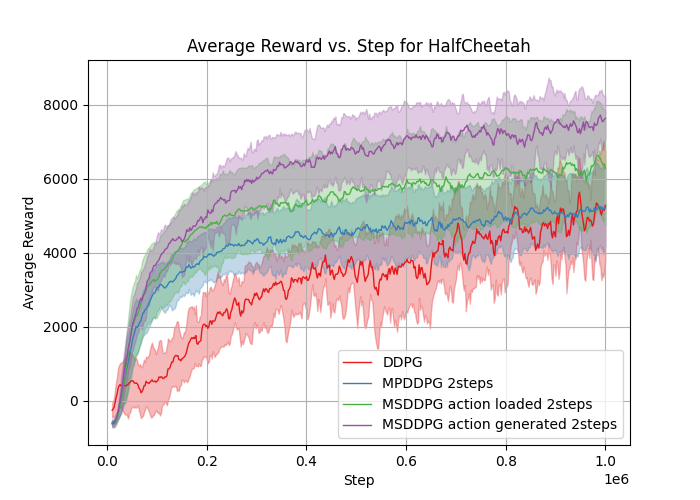}
	\label{fig:ddpghalf2}
	\end{minipage}
	\hfill
    \begin{minipage}{0.315\linewidth}
	\centering
	\includegraphics[width=\linewidth]{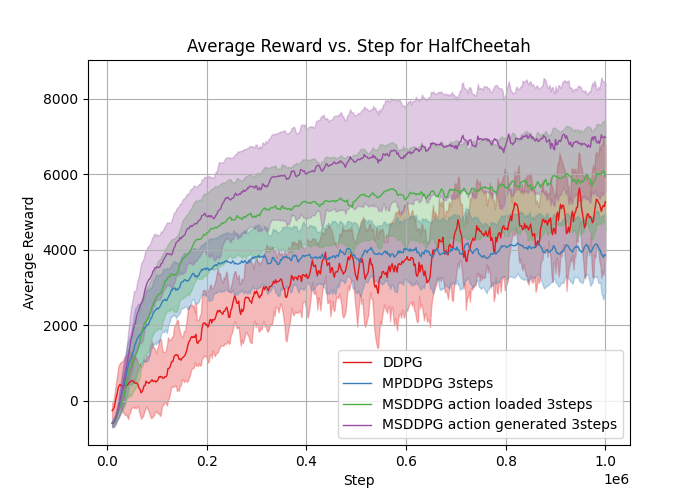}
	\label{fig:ddpghalf3}
        \end{minipage}
	\hfill
    \begin{minipage}{0.315\linewidth}
	\centering
	\includegraphics[width=\linewidth]{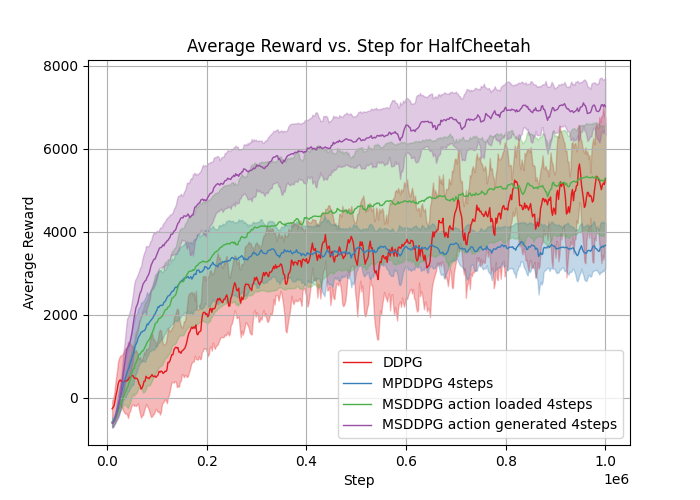}
	\label{fig:ddpghalf4}
	\end{minipage}
	\caption{Performance comparison of various DDPG algorithms in HalfCheetah-v4 environment. }
	\label{fig:ddpg_half}
\end{figure*}

\begin{figure*}
\centering
    \begin{minipage}{0.315\linewidth}
	\centering
	\includegraphics[width=\linewidth]{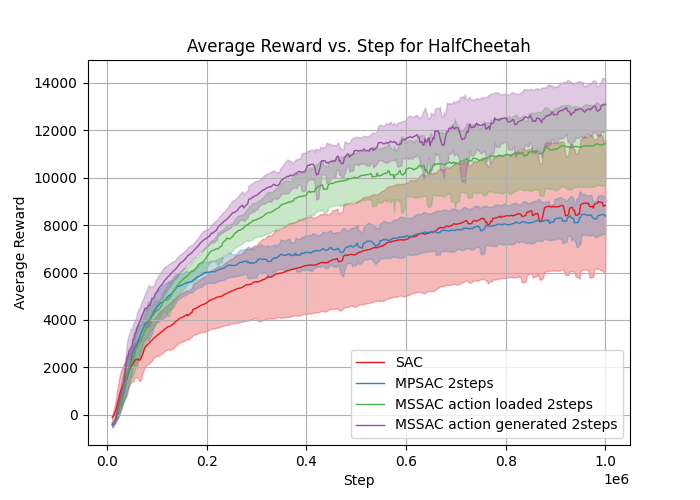}
	\label{fig:sachalf2}
	\end{minipage}
	\hfill
    \begin{minipage}{0.315\linewidth}
	\centering
	\includegraphics[width=\linewidth]{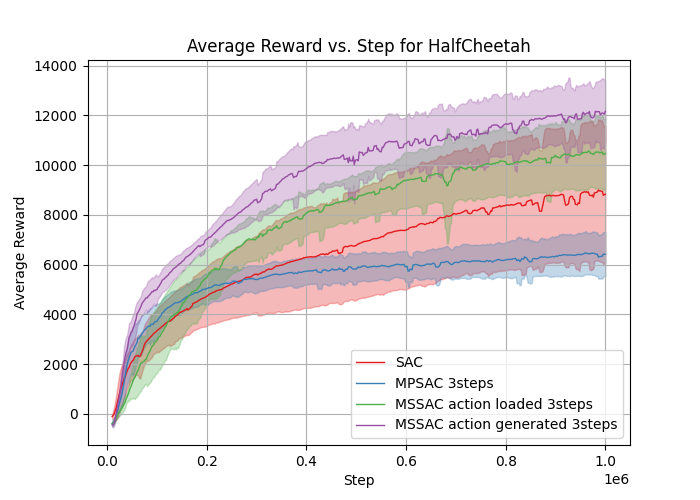}
	\label{fig:sachalf3}
        \end{minipage}
	\hfill
    \begin{minipage}{0.315\linewidth}
	\centering
	\includegraphics[width=\linewidth]{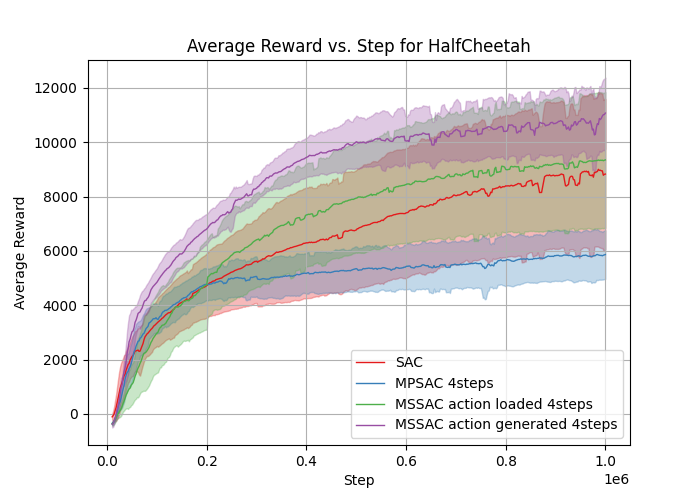}
	\label{fig:sachalf4}
	\end{minipage}
	\caption{Performance comparison of various SAC algorithms in HalfCheetah-v4 environment.}
	\label{fig:sac_half}
\end{figure*}

\end{document}